\documentclass[letterpaper, 10 pt, conference]{ieeeconf}  

\IEEEoverridecommandlockouts                              

\usepackage{amsmath,amsfonts,amssymb}
\usepackage{prettyref}
\usepackage{mathrsfs}
\usepackage{graphicx}
\usepackage{wrapfig}
\usepackage{subfig}
\usepackage{MnSymbol}
\usepackage{multirow}
\usepackage{booktabs}
\usepackage{algorithm}
\usepackage[noend]{algpseudocode}
\usepackage[table]{xcolor}
\usepackage{cellspace}
\usepackage{mathtools}
\usepackage{sidecap}

\newtheorem{theorem}{Theorem}[section]
\newtheorem{lemma}[theorem]{\TE{Lemma}}

\algnewcommand{\LineComment}[1]{\State \(\triangleright\) #1}
\algdef{SE}[DOWHILE]{Do}{doWhile}{\algorithmicdo}[1]{\algorithmicwhile\ #1}
\newcommand*{\colorboxed}{}
\def\colorboxed#1#{%
  \colorboxedAux{#1}%
}
\newcommand*{\colorboxedAux}[3]{%
  \begingroup
    \colorlet{cb@saved}{.}%
    \color#1{#2}%
    \boxed{%
      \color{cb@saved}%
      #3%
    }%
  \endgroup
}

\newrefformat{Fig}{Figure~\ref{#1}}
\newrefformat{fig}{Figure~\ref{#1}}
\newrefformat{par}{Section~\ref{#1}}
\newrefformat{appen}{Appendix~\ref{#1}}
\newrefformat{sec}{Section~\ref{#1}}
\newrefformat{sub}{Section~\ref{#1}}
\newrefformat{table}{Table~\ref{#1}}
\newrefformat{alg}{Algorithm~\ref{#1}}
\newrefformat{Alg}{Algorithm~\ref{#1}}
\newrefformat{Def}{Definition~\ref{#1}}
\newrefformat{Thm}{Theorem~\ref{#1}}
\newrefformat{Lem}{Lemma~\ref{#1}}
\newrefformat{step}{Step~\ref{#1}}
\newrefformat{ln}{Line~\ref{#1}}
\newrefformat{eq}{Equation~\ref{#1}}
\newrefformat{pb}{Problem~\ref{#1}}
\newrefformat{it}{Item~\ref{#1}}
\newrefformat{te}{Term~\ref{#1}}
\def\Eqref Eq:#1:{\eqref{eq:#1}}
\newrefformat{Eq}{Equation~\Eqref#1:}

\newcommand{\E}[1]{\mathbf{#1}}
\newcommand{\TE}[1]{\textbf{#1}}

\newcommand{\TWO}[2]{\left(\setlength{\arraycolsep}{1pt}\begin{array}{cc}{#1}, & {#2}\end{array}\right)}

\newcommand{\fmin}[1]{\underset{#1}{\E{min}}\;}
\newcommand{\fmax}[1]{\underset{#1}{\E{max}}\;}
\newcommand{\argmin}[1]{\underset{#1}{\E{argmin}}\;}
\newcommand{\argminP}[1]{\E{argmin}\;}

\newcommand{\argmaxP}[1]{\E{argmax}\;}
\newcommand{\ST}{\E{s.t.}\;}

\definecolor{darkgreen}{HTML}{186a3b}
\newcommand{\IND}[1]{\mathbb{I}[#1]}

\title{\Large\bf Decision Making in Joint Push-Grasp Action Space for Large-Scale Object Sorting\vspace{-5px}}
\author{Zherong Pan and Kris Hauser$^\dagger$\vspace{-20px}  \\
\thanks{$^\dagger$ Zherong Pan and Kris Hauser are with the Department of Computer Science, University of Illinois at Urbana-Champaign. {\tt\small \{zherong,kkhauser\}@illinois.edu}}}

\newif\ifarxiv
\arxivtrue
\begin{document}
\maketitle
\thispagestyle{empty}
\pagestyle{empty}

\begin{abstract}
We present a planner for large-scale (un)labeled object sorting tasks, which uses two types of manipulation actions: overhead grasping and planar pushing. The grasping action offers completeness guarantee under mild assumptions, and planar pushing is an acceleration strategy that moves multiple objects at once. Our main contribution is twofold: (1) We propose a bilevel planning algorithm. Our high-level planner makes efficient, near-optimal choices between pushing and grasping actions based on a cost model. Our low-level planner computes one-step greedy pushing or grasping actions. (2) We propose a novel low-level push planner that can find one-step greedy pushing actions in a semi-discrete search space. The structure of the search space allows us to efficient We show that, for sorting up to $200$ objects, our planner can find near-optimal actions with $10$ seconds of computation on a desktop machine.
\end{abstract}

\section{\label{sec:intro}Introduction}
Countless object sorting machines have been designed over the past century. The robustness of these machines are high enough to be used as a part of manufacturing process. Early systems use pure mechanical gadgets to force objects into separate buckets according to their shapes~\cite{leopoldo1961fruit,odquist1943sorting}. In addition to robustness, the efficacy of these mechanical systems are rather high, allowing multiple objects to be sorted in parallel.  But warehouse automation and service robotics require sorting objects according to visual features, such as the printed address on a package or object color. The vast majority of sorting robots \cite{wurman2016amazon,yu2016summary} solely rely on grasping actions and treat multiple objects in a serial manner. This design choice is largely due to the robustness of grasping to uncertainties in perception and execution. However, serial object grasping does not even reach a fraction of the throughput of purely mechanical gadgets.

Planar pushing is a promising direction to achieve more efficient sorting, because many objects can be moved at once~\cite{8793946,Song2019}. However, planning for pushing is far more complex that that for grasping for three reasons. First, the action space of planar pushing is continuous, involving the pusher's initial orientation and pushing direction, while the action space of overhead grasping is (relatively) discrete. Second, the dynamics of pushing are complex and uncertain, even in the single-object case, since the continuous pressure force distribution between the object and the ground is unknown~\cite{7487155}. Third, although Akella and Mason~\cite{219923} showed that a single object can be pushed to an arbitrary pose, this has not yet been proven for multi-object pushing. Recent learning-based methods \cite{8793946,Song2019,Zeng2018} reformulate an object sorting problem as an optimization problem by defining a cost model and use stochastic search to reduce the cost. But these methods are incomplete, and it is difficult to either analyze or anticipate their failure cases.

\begin{figure}[ht]
\centering
\includegraphics[width=0.95\linewidth]{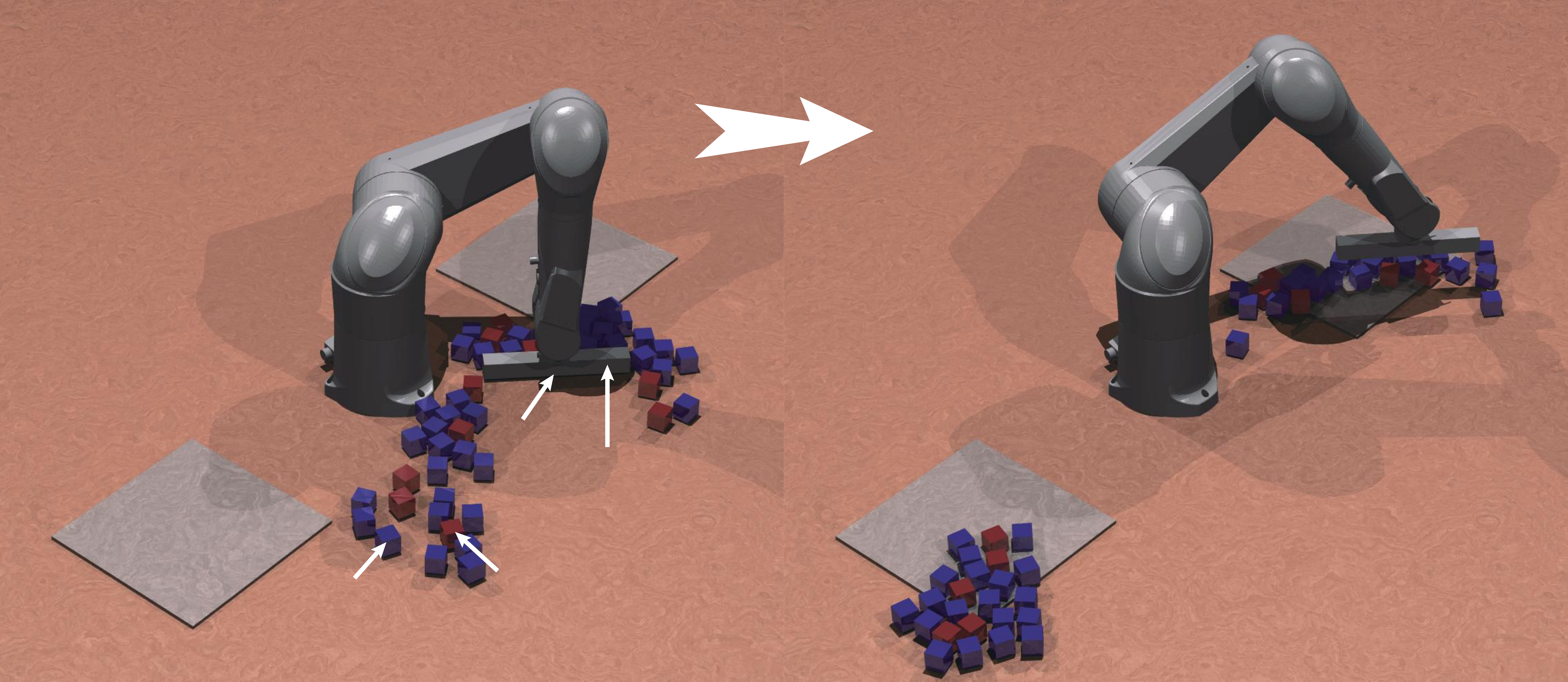}
\put(-138,81){\textcolor{white}{\tiny{(a)}}}
\put(-100,81){\textcolor{white}{\tiny{(b)}}}
\put(-225,30){\textcolor{white}{\tiny{Target Region 1}}}
\put(-140,65){\textcolor{white}{\tiny{Target Region 2}}}
\put(-170,33){\textcolor{white}{\tiny{Sucker}}}
\put(-150,28){\textcolor{white}{\tiny{Pusher}}}
\put(-215,10){\textcolor{white}{\tiny{Object Category 1}}}
\put(-160,10){\textcolor{white}{\tiny{Object Category 2}}}
\caption{\label{fig:teaser} (a): An illustration of our problem setting. The goal is for the same number of red and blue objects to fall into the 1st and 2nd target region. (b): After 3 pushing actions and 1 grasping action, the objects have been approximately sorted.}
\vspace{-5px}
\end{figure}
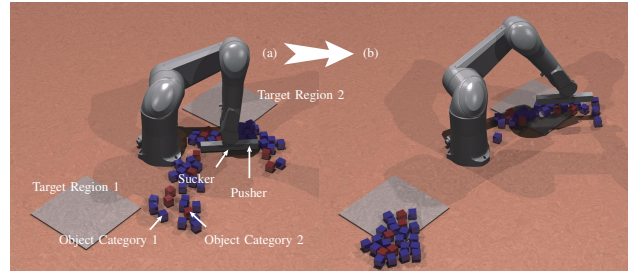

\TE{Main Results:} Our main contribution is a bilevel motion planning algorithm that can efficiently make decisions in joint push-grasp action space. The efficacy of our method is due to two novel techniques. (1) We decompose the responsibility between the high- and low-level planner, such that low-level planner can efficiently determine one-step greedy grasping or pushing actions and high-level planners make binary choices between grasping and pushing actions over multiple steps. Since the high-level planner only considers greedy actions, the branching factor is significantly reduced and search over multiple steps become practical. (2) We take mild assumptions in the low-level push planner, so that finding the optimal pushing action becomes a numerical optimization with piecewise quadratic objective functions. As a result, the optimal pushing action can be found via quadratic-piece enumeration, and costly global optimization is avoided.

Compared with learning-based methods \cite{8793946,Song2019,Zeng2018}, we can provide completeness guarantee with the help of the grasp action under mild assumptions that feasibility is not violated by non-prehensile manipulations. (As indicated in \cite{TanYu19ISRR}, non-prehensile manipulations can move objects into unreachable regions of the robot arm, making problem infeasible.) Our method is solely analytical and does not require hyperparameter tuning. We evaluated our synergetic planner on both labeled and unlabeled tasks of sorting 50-200 objects. The results show that our method can benefit from pushing actions to achieve up to $10\times$ speedup in terms of execution time, as compared with our method using grasping actions alone. And the computational time to solve for the next action is within $10$ seconds on a desktop machine.
\section{\label{sec:related}Related Work}
We review related work in multi-object manipulation, push planning, and grasp planning.

\TE{Multi-object Manipulation} allows the robot to move multiple objects simultaneously in order to accomplish a task. Typical tasks involve object sorting \cite{8793946,Song2019}, clutter removal \cite{TanYu19ISRR,TerryRuss2020WAFR}, object placement \cite{cosgun2011push}, and object singulation \cite{Zeng2018,8793972}. We notice several common design choices in these methods. First, all these methods are restricted to 2D workspaces by assuming that the gripper always reaches objects from overhead. Our method also uses this simplification. Second, most of the proposed methods use a single action, either grasping or pushing. An exception is made in \cite{Zeng2018}, where objects are singulated by pushing actions and then grasped, which is similar to our planner. But the pushing action in \cite{Zeng2018} is used as a grasping auxiliary, and objects are always transferred using grasping actions, while we allow objects to be transferred by both pushing and grasping. Finally, our algorithm is designed to be analytic and parameter-free, whereas prior learning-based works \cite{TanYu19ISRR,TerryRuss2020WAFR,Zeng2018} are sensitive to parameters. On the other hand, we assume perfect perception of object positions, while learning-based methods can deal-with sensing certainty by training a visuomotor policy in an end-to-end manner.

\TE{Grasp Planning} is relatively simple in our problem as we assume the use of a sucker. Most prior works assume more dexterous grippers such as the parallel jaw gripper \cite{mahler2017binpicking} and multi-fingered grippers \cite{miller2004graspit}. Parallel jaw gripper is available at a low cost and thus assumed in multi-object manipulation planners, e.g. \cite{Zeng2018,8793972}, but gripper feasibility can pose a major problem when objects are densely cluttered. In their most recent work, Mahler \cite{mahler2019learning} used both parallel jaw gripper and sucker mounted on two arms. They argued that the sucker might fail on certain materials such as hairy deformable objects and objects made of porous media. In applications like warehouse automation, however, this problem can be avoided by packing objects into a boxes.

\TE{Push Planning} is a well-studied problem if there is one single object. Prior work \cite{goyal1991planar} showed that the object motion under-pushing can be approximated quasistatically by modeling the limiting surface of contact wrenches. The feasibility of posing a single object using pushing has been proved in \cite{219923}. Another proof is presented in \cite{zhou2019pushing} by showing that planar pushing is differential flat. However, the object dynamics, motion planning algorithm, and feasibility when pushing multiple objects are still open problems. We propose an aggressively simplified multi-object prediction model. We show that this model allows efficient computation of the one-step greedy pushing action, while the approximation error is acceptable for the purpose of accelerating object sorting tasks.
\section{\label{sec:prob}Object Sorting Problem}
In this section, we formulate large-scale (un)labeled object sorting tasks. We assume there are $N$ planar objects with center-of-mass at $\E{o}_{1,\cdots,N}$. The planar assumption is used by our low-level grasp planner to enable overhead grasping using suckers. It is also used by the push planner to analyze object configurations in the 2D projected workspace. These objects are divided into $C$ categories and each object is assigned a category label $l_i\in\{1,\cdots,C\}$. Our problem definition unifies unlabeled object sorting when $C=1$ and fully labeled object sorting when $C=N$.

We further assume that there are $T$, pairwise disjoint virtual target regions, where each region is also represented as a convex polygon. These regions are virtual and not marked by any physical objects, so that objects will not be blocked when pushed. The convexity of regions is also required by the push planner to predict the result of a potential pushing action. In addition, the regions must be disjoint for the completeness of one-step grasping actions. We denote the closed convex set of the $j$th target region as $\E{t}_j$. Each $\E{t}_j$ has a capacity for each object category, denoted as $c_{1,\cdots,C}(\E{t}_j)$. The goal of an object sorting task is to move all $\E{o}_i$ such that the following two conditions hold:
\begin{align}
\label{eq:cond1}
\sum_{j=1}^T&\IND{\E{o}_i\in\E{t}_j}=1
\quad\forall i=1,\cdots,N   \\
\label{eq:cond2}
\sum_{l_i=k}&\IND{\E{o}_i\in\E{t}_j}\leq c_k(\E{t}_j)
\quad\forall j=1,\cdots,T\quad k=1,\cdots,C,
\end{align}
where $\IND{\bullet}$ is the indicator function. The first equation implies that each object must fall inside one of the target regions. The second equation implies that, in a certain region, the number of objects of a certain category does not exceed the capacity of that region. 

\subsection{3D Workspace}
We conduct experiments in the 3D workspace as shown in \prettyref{fig:teaser}, but our planner performs all the computations in the projected 2D workspace. The 2D-to-3D gap is closed by using a reachability analysis of the gripper. We precompute the inverse kinematics for each uniformly sampled planar position, and approximate the inverse kinematics in between samples using bilinear interpolation. The neighboring samples are connected such that a planar trajectory can be globally resolved using the algorithm proposed in \cite{hauser2018global}, which is important for realizing a pushing action. Our planner will use this reachability map in three ways:
\begin{itemize}
\item Function $\E{reach}(\E{x})$ checks whether $\E{x}$ can be reached.
\item Function $\E{traj}(\E{x},\E{y})$ finds a trajectory from $\E{x}$ to $\E{y}$.
\item Function $\E{range}(\E{x},\E{d})$ returns a pair of distances $(a,b)$ that defines the maximal resolvable push from $\E{x}$ along $\E{d}$, i.e. a push from $\E{x}+a\E{d}$ to $\E{x}+b\E{d}$ is the longest, globally resolvable push ($a,b$ can be negative).
\end{itemize}

\subsection{Overview}
Finding reasonable pushing or grasping actions is challenging due to a continuous decision space and a long planning horizon. Although the decision space for grasping actions is discrete, push planner must search over a continuous space of the pusher's position, orientation, and moving distance. If the continuous space is exhaustively discretized, then the branching factor can be prohibitively high. On the other hand, our experiments show that solving a sorting task can require up to 50 grasping or pushing actions. If a motion plan is only accepted when it successfully accomplish a sorting task, then full-horizon planning is required, which is impractical considering the high branching factor.

We combine two ideas to design a practical bilevel planning algorithm. First, we introduce a cost function that measures the closedness between an arbitrary configuration and a final, sorted configuration. As a result, our high-level planner can work in a receding-horizon mode guided by the cost function. Second, we significantly reduce the branching factor by only considering one-step greedy actions. In other words, our high-level planner only chooses the type of actions (grasping or pushing), while we use two low-level planners to ensure that the chosen action leads to the highest reduction in the cost function among all the actions of the same type.

Intuitively, our cost function $J$ sums over the distances between objects and target regions. This cost is zero if and only if sorting is successful. Among all possible object-to-target-region pairings, we choose the one with lowest cost value, which amounts to the following optimization:
\begin{equation}
\begin{aligned}
\label{eq:cost}
J(\E{o}_i)=\fmin{b_{ij}\in\{0,1\}}&
\sum_{i=1}^N\sum_{j=1}^T b_{ij}\E{dist}(\E{o}_i,\E{t}_j)    \\
\ST&\sum_{j=1}^Tb_{ij}=1\quad\sum_{l_i=k}b_{ij}=c_k(\E{t}_j),
\end{aligned}
\end{equation}
where $\E{dist}$ is the Euclidean distance between a point and a convex polygon. Computing $J$ amounts to solving an optimal assignment problem, for which the Hungarian algorithm can be used at a computational cost of $\mathcal{O}(N^3)$, by introducing dummy variables to absorb the capacity constraints. Obviously, $J(\E{o}_i)=0$ if and only if the two conditions in \prettyref{eq:cond1},\ref{eq:cond2} hold, but the function $J$ allows us to monitor progress and compare different planning algorithms. 

In rest of the paper, we first introduce the low-level grasp planner (\prettyref{sec:grasp}) and the push (\prettyref{sec:push}) planner. We then introduce a single high-level planner (\prettyref{sec:high}) that chooses greedy actions over multiple steps to minimize the cost function in a receding-horizon manner as outlined in \prettyref{alg:high_level}. The assumptions and completeness guarantees are provided in \ifarxiv \prettyref{sec:feasible} \else our extended report \cite{} \fi.
\section{\label{sec:grasp}Low-Level Grasp Planner}
We show that the one-step greedy grasping action can be found by solving a mixed-integer linear programming (MILP). We first define the radius of an object. If we compute a bounding circle for the $i$th object centered at $\E{o}_i$ with radius $r_i$, then we define $R=\fmax{i=1,\cdots,N}r_i$. We sample a set of potential positions $\E{p}_{mn}$ to put the grasped object, and we assume uniform sampling with a spacing equals to $\sqrt{2}R$, i.e. $\E{p}_{mn}\triangleq\TWO{\sqrt{2}Rm}{\sqrt{2}Rn}$. We define the set of reachable $\sqrt{2}R$-spaced samples that fall inside the $j$th target region as follows:
\begin{align*}
\E{S}_j=\{\E{p}_{mn}|B_{2R}(\E{p}_{mn})\subseteq\E{t}_j\wedge\E{reach}(\E{p}_{mn})=1\},
\end{align*}
where $B_R(\E{p}_{mn})$ is a ball centered at $\E{p}_{mn}$ with radius $R$.

To find the one-step greedy grasping action that reduces the cost as much as possible, we introduce binary variables $b_{ij}$ as in the cost model, where $b_{ij}=1$ implies that $\E{o}_i$ is not the object to be grasped and it is assigned to $\E{t}_j$. We further introduce another set of binary variables $p_{mnj}$ for each $\E{p}_{mn}$ and $j=1,\ldots,T$, where $p_{mnj}=1$ implies that an object will be grasped and put to the sampled location $\E{p}_{mn}$ and this grasped object will be assigned to $\E{t}_j$. After solving for $b_{ij},p_{mnj}$, we can identify the object $\E{o}_i$ to be grasped if $\sum_{j=1}^Tb_{ij}=0$ and we will move it to $\E{p}_{mnj}$ if $p_{mnj}=1$. Finally, we compute the gripper trajectory by calling $\E{traj}(\E{o}_i,\E{p}_{mn})$. We solve for $b_{ij},p_{mnj}$ using the following MILP:
\begin{small}
\begin{equation}
\begin{aligned}
\label{eq:MILP}
&\argmin{b_{ij}, p_{mnj}\in\{0,1\}} \mathcal{O}    \\
\ST&J_{post}\leq J(\E{o}_i) \\
&\sum_{i=1}^N\sum_{j=1}^T b_{ij}=N-1\quad
\sum_{mn}\sum_{j=1}^T p_{mnj}=1   \\
&\sum_{j=1}^Tb_{ij}\leq1\quad\sum_{l_i=k}b_{ij}=c_k(\E{t}_j)    \\
J_{post}\triangleq&\sum_{i=1}^N\sum_{j=1}^T b_{ij}\E{dist}(\E{o}_i,\E{t}_j)+
\sum_{mn}\sum_{j=1}^T p_{mnj}\E{dist}(\E{p}_{mn},\E{t}_j),
\end{aligned}
\end{equation}
\end{small}
where $J_{post}$ is the post-grasping cost. We have used three types of constraints. First, we ensure that cost is monotonically reduced. Second, we ensure that only one object will be grasped and the object will be put to only one sampled location. Finally, we have the assignment constraints (each object can only be assigned to one target region) and capacity constraints. The objective function $\mathcal{O}$ can take multiple forms. If we want to reduce the total cost as much as possible, then $\mathcal{O}=J_{post}$, and we denote the resulting grasping action as $\mathcal{G}(\E{o}_i,\E{p}_{mn})$. In this case, the function grasp(STATE) in \prettyref{alg:high_level} returns $\{\mathcal{G}(\E{o}_i,\E{p}_{mn})\}$ and contributes 1 to the branching factor. If we want to reduce the cost related to a single target region, e.g. $\E{t}_j$, then we can define:
\begin{align*}
\mathcal{O}=\sum_{i=1}^N b_{ij}\E{dist}(\E{o}_i,\E{t}_j)+\sum_{mn} p_{mnj}\E{dist}(\E{p}_{mn},\E{t}_j), 
\end{align*}
and we denote the resulting grasping action as $\mathcal{G}_j(\E{o}_i,\E{p}_{mn})$. In this case, grasp(STATE) returns $\{\mathcal{G}_1,\cdots,\mathcal{G}_T\}$ and contributes $T$ to the branching factor. Finally, we show in \ifarxiv \prettyref{sec:feasible} \else our extended report \cite{} \fi that $J_{cost}$ can be monotonically reduced to zero under mild assumptions, thereby providing a completeness guarantee.
\setlength{\textfloatsep}{4pt}
\begin{algorithm}[ht]
\caption{\label{alg:high_level} High-Level Planner}
\begin{small}
\begin{algorithmic}[1]
\Require Initial state $\text{STATE}\gets\{\E{o}_i\}$, max horizon $H$
\State Initialize stack $\text{STACK}\gets\{\text{STATE}\}$
\State Best action $\text{ACTION}^*\gets\text{None}$, $J_{rate}^*\gets\infty$
\While{$\text{STACK}$ not empty}
\State $\text{STATE}\gets\text{pop}(\text{STACK})$
\State $\{\text{ACTION}\}\gets\text{grasp}(\text{STATE})\cup\text{push}(\text{STATE})$
\For{$\text{ACTION}\in\{\text{ACTION}\}$}
\State \label{ln:sim}$\text{STATE}^+=\text{simulate}(\text{STATE},\text{ACTION})$
\If{$\text{horizon}(\text{STATE}^+)<H$}
\State $\text{STACK}\gets\text{STACK}\cup\{\text{STATE}^+\}$
\ElsIf{$J_{rate}(\text{STATE}^+)<J_{rate}^*$}
\State $\text{ACTION}^*\gets \text{backTrace}(\text{STATE}^+)$
\State $J_{rate}^*\gets J_{rate}(\text{STATE}^+)$
\EndIf
\EndFor
\EndWhile
\State Return $\text{ACTION}^*$
\end{algorithmic}
\end{small}
\end{algorithm}
\section{\label{sec:push}Low-Level Push Planner}
In this section, we propose a method to compute the one-step greedy pushing action. This problem is challenging as we are making decisions in a continuous action space that involves the pusher's location, pushing direction, and pushing distance. Indeed, even predicting the single object motion during pushing is non-trivial \cite{7487155}. To analyze multi-object motions, our method is based on the following two assumptions similar to \cite{TerryRuss2020WAFR}:
\begin{itemize}
    \item The pusher is rectangular and the pushing direction is orthogonal to the pusher.
    \item During pushing, objects will only translate along the pushing direction, no rotation or perpendicular translation will happen.
\end{itemize}
We illustrate some key notions in \prettyref{fig:push}(a). We assume that the pusher can only move in one of $D$ directions. For each pushing direction $\E{d}$, we define the affected region (gray) as the region formed by sweeping the pusher along $\E{d}$. Any object that falls entirely inside this region (blue) will be considered affected. There are boundary cases when objects fall partially in this region (red). We assume that objects of boundary cases will not be affected by the pushing action. Our push planner consists of two steps. First, we show that there are only discrete number of possible pusher locations that can be enumerated. Second, for each pusher's location, we compute the optimal pushing distance $d$.

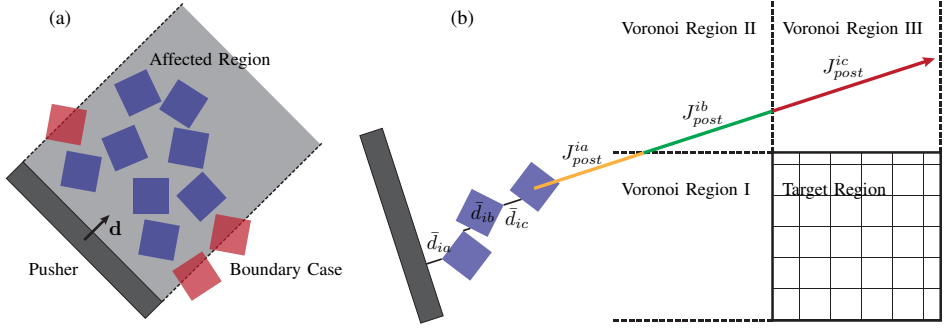
\begin{SCfigure*}[][t]
\centering
\scalebox{0.76}{\includegraphics[width=0.95\textwidth]{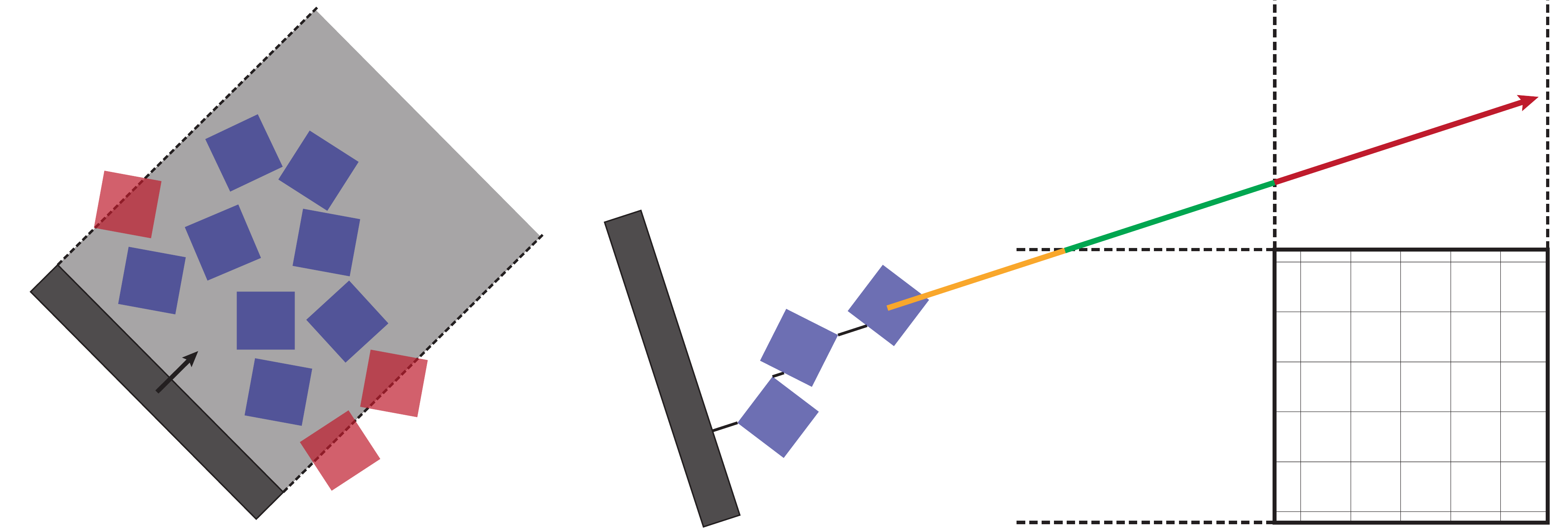}
\put(-450,150){(a)}
\put(-460,25){\small{Pusher}}
\put(-360,25){\small{Boundary Case}}
\put(-400,130){\small{Affected Region}}
\put(-420,45){$\E{d}$}
\put(-250,150){(b)}
\put(-85,65){\small{Target Region}}
\put(-165,65){\small{Voronoi Region I}}
\put(-165,145){\small{Voronoi Region II}}
\put(-85,145){\small{Voronoi Region III}}
\put(-262,38){$\bar{d}_{ia}$}
\put(-240,53){$\bar{d}_{ib}$}
\put(-223,50){$\bar{d}_{ic}$}
\put(-195,83){$J_{post}^{ia}$}
\put(-135,103){$J_{post}^{ib}$}
\put(-65,126){$J_{post}^{ic}$}}
\vspace{-10px}
\caption{\small{\label{fig:push} Illustration of our simplified kinematic model, which predicts the change of $J(\E{o}_i)$ as a function of pushing distance. We ignore the rotation of objects and only consider their linear motions along the pushing direction. Under this assumption, the cost function of each object is piecewise quadratic, where the quadratic pieces are dictated by the Voronoi regions of the target area.}
\vspace{-10px}}
\end{SCfigure*}

\subsection{Locating the Pusher}
For a pushing direction $\E{d}$, its orthogonal direction is denoted as $\E{d}^\perp$. A pusher's location is expressed as $\alpha\E{d}+\beta\E{d}^\perp$. We compute the two coefficients $\alpha,\beta$ by sorting objects' locations along $\E{d}$ and $\E{d}^\perp$. Since we assume that an object $\E{o}_i$ is a convex polygon, we can define its vertices as $\E{v}_i^1,\cdots,\E{v}_i^{V(\E{o}_i)}$, where $V(\E{o}_i)$ is the number of vertices in $\E{o}_i$. We then define the four supports of $\E{o}_i$ along $\E{d}$ and $\E{d}^\perp$ as:
\begin{align*}
\E{d}_{min/max(i)}&=\fmin{k=1,\cdots,V(\E{o}_i)}/\fmax{k=1,\cdots,V(\E{o}_i)}<\E{d},\E{v}_i^k> \\
\E{d}_{min/max(i)}^\perp&=\fmin{k=1,\cdots,V(\E{o}_i)}/\fmax{k=1,\cdots,V(\E{o}_i)}<\E{d}^\perp,\E{v}_i^k>.
\end{align*}
Similarly, the pusher is rectangular so it has four supports denoted as $\E{d}_{min/max(p)}$ and $\E{d}_{min/max(p)}^\perp$. We then record all $\beta$ values satisfying $\beta+\E{d}_{min/max(p)}^\perp=\E{d}_{min/max(i)}^\perp$ for some $i$ and we sort these key $\beta$ values in ascending order denoted as $\beta_1\leq\beta_2\leq\cdots\leq\beta_{4N}$, where there are at most $4N$ cases. When the pusher moves between $\beta_n$ and $\beta_{n+1}$ along $\E{d}^\perp$, the set of affected objects is invariant and denoted as:
\begin{small}
\begin{align*}
\mathcal{A}_n^\perp=\{\E{o}_i|
\beta_n+\E{d}_{min(p)}^\perp\leq 
\E{d}_{min(i)}^\perp\leq\E{d}_{max(i)}^\perp\leq 
\beta_{n+1}+\E{d}_{max(p)}^\perp\}.
\end{align*}
\end{small}
We repeat this procedure along $\E{d}$ to define the $4N$ key values $\alpha_1\leq\alpha_2\leq\cdots\leq\alpha_{4N}$, and the set of affected objects:
\begin{small}
\begin{align*}
\mathcal{A}_m=\{\E{o}_i|\alpha_m+\E{d}_{max(p)}\leq \E{d}_{min(i)}\}.
\end{align*}
\end{small}
Note that the definition of $\mathcal{A}_m$ is different from $\mathcal{A}_n^\perp$ in that we only consider objects in front of the pusher, as illustrated in In \prettyref{fig:push} (a). Finally, we define an additional set of objects overlapping the pusher as:
\begin{small}
\begin{align*}
\mathcal{I}_m=\{\alpha_m+\E{d}_{min(p)}\leq 
\E{d}_{min(i)}\leq\E{d}_{max(i)}\leq 
\alpha_{m+1}+\E{d}_{max(p)}\}.
\end{align*}
\end{small}
Given these notations, we summarize that a possible pusher location $\alpha\E{d}+\beta\E{d}^\perp$ must satisfy the following conditions:
\begin{align*}
&\alpha\in[\alpha_m,\alpha_{m+1}]\quad
\beta\in[\beta_n,\beta_{n+1}]   \\
&\mathcal{A}_m\cap\mathcal{A}_n^{\perp}\neq\emptyset\quad
\mathcal{I}_m\cap\mathcal{A}_n^{\perp}=\emptyset,
\end{align*}
where there are at most $16N^2$ choices. To further reduce the computational cost, we can remove one of the case, if two cases have the same set of affected objects, i.e. $\mathcal{A}_m\cap\mathcal{A}_n^{\perp}$.

\setlength{\textfloatsep}{4pt}
\begin{algorithm}[ht]
\caption{\label{alg:dbar} Computing the compression distance $\bar{d}_i$}
\begin{small}
\begin{algorithmic}[1]
\State $\bar{d}_i\gets\E{d}_{min(i)}-\alpha-\E{d}_{max(p)}$
\For{Each vertices $\E{v}_i^k$, $k=1,\cdots,V(\E{o}_i)$}
\State Shoot a ray from $\E{v}_i^k$ along $-\E{d}$, record first intersection.
\If{Ray intersects object $\E{o}_j$ after traveling $\bar{d}_i^k$}
\If{$\E{o}_j\in\mathcal{A}_m\cap\mathcal{A}_n$}
\State $\bar{d}_i\gets\E{min}(\bar{d}_i,\bar{d}_i^k+\bar{d}_j)$\Comment{Recursion}
\EndIf
\EndIf
\EndFor
\State Return $\bar{d}_i$
\end{algorithmic}
\end{small}
\end{algorithm}
\subsection{Finding the Optimal Pushing Distance}
For a given $\E{d},\alpha\in\mathcal{A}_m,\beta\in\mathcal{A}_n^\perp$, we plan the optimal pusher distance that reduces $J(\E{o}_i)$ the most. We first solve \prettyref{eq:cost} to find $b_{ij}=1$, i.e. an affected object $\E{o}_i$ is assigned to the target region $\E{t}_j$. If we move the pusher by distance $d$, then we need to compute the following post-pushing cost function:
\begin{align*}
J_{post}^i(d)\triangleq\E{dist}(\E{o}_i(d),\E{t}_j).
\end{align*}
If $J_{post}^i(d)$ can be expressed analytically, then we can find the optimal pushing distance by solving:
\begin{equation}
\begin{aligned}
\label{eq:pushOpt}
d^*=\argmin{d}&\mathcal{O} \\
\ST&\sum_{\E{o}_i\in\mathcal{A}_m\cap\mathcal{A}_n^\perp}J_{post}^i(d)\leq J(\E{o}_i) \\
&d\in\E{range}(\alpha\E{d}+\beta\E{d}^\perp,\E{d}),
\end{aligned}
\end{equation}
where we have added a constraint to ensure monotonic cost reduction. Similar to the case with grasping actions, the objective function $\mathcal{O}$ can take multiple forms. Similar to the grasp planner, the push planner can be used in two ways. If we want to reduce the overall cost function, we can set $\mathcal{O}=\sum_{\E{o}_i\in\mathcal{A}_m\cap\mathcal{A}_n^\perp}J_{post}^i(d)$ and denote the resulting pushing action as $\mathcal{P}(\E{d},\alpha,\beta,d^*)$. In this case push(STATE) in \prettyref{alg:high_level} returns $\{\mathcal{P}\}$ and contributes to the 1 branching factor. If we want to reduce the cost related to a single target region, e.g. $\E{t}_j$, we can set $\mathcal{O}=\sum_{\E{o}_i\in\mathcal{A}_m\cap\mathcal{A}_n^\perp\wedge b_{ij}=1}J_{post}^i(d)$ and denote the resulting pushing action as $\mathcal{P}_j(\E{d},\alpha,\beta,d^*)$. In this case push(STATE) in \prettyref{alg:high_level} returns $\{\mathcal{P}_1,\cdots,\mathcal{P}_T\}$ and contributes $T$ to the branching factor.

\setlength{\textfloatsep}{4pt}
\begin{algorithm}[ht]
\caption{\label{alg:push} Push planner}
\begin{small}
\begin{algorithmic}[1]
\State Solution $<\alpha^*,\beta^*,d^*>\gets\text{None}$, best $\mathcal{O}^+\gets\infty$
\State Compute all possible ranges $\{[\alpha_m,\alpha_{m+1}]\}$ and $\{[\beta_n,\beta_{n+1}]\}$
\For{Each $[\alpha_m,\alpha_{m+1}]\in\{[\alpha_m,\alpha_{m+1}]\}$}
\For{Each $[\beta_n,\beta_{n+1}]\in\{[\beta_n,\beta_{n+1}]\}$}
\State $\mathcal{O}(d)=0$\Comment{Build objective}
\For{$i=1,\cdots,N$}
\LineComment{Only consider objects in affected region}
\LineComment{Only consider objects in front of pusher}
\If{$i\in\text{affected}([\alpha_m,\alpha_{m+1}],[\beta_n,\beta_{n+1}])$}
\State Compression distance $\bar{d}_i$ (\prettyref{alg:dbar})
\State $\mathcal{O}(d)\gets\mathcal{O}(d)+J_{post}^i(d)$
\EndIf
\EndFor
\State Solve \prettyref{eq:pushOpt} for $\mathcal{O}^+,d^+$\Comment{Minimize objective}
\If{$\mathcal{O}^+<\mathcal{O}^*$}
\State $<\alpha^*,\beta^*,d^*>\gets<\frac{\alpha_m+\alpha_{m+1}}{2},\frac{\beta_m,\beta_{m+1}}{2},d^+>$
\State $\mathcal{O}^*\gets\mathcal{O}^+$
\EndIf
\EndFor
\EndFor
\State Return $<\alpha^*,\beta^*,d^*>$
\end{algorithmic}
\end{small}
\end{algorithm}
To solve for the global minima of the above 1D optimizations analytically, we show that each $J_{post}^i(d)$ is piecewise quadratic and so is their summation. As a result, the 1D optimization can be solved by enumerating and finding the global minima of each quadratic piece. The first quadratic piece is denoted as the void piece with length $\bar{d}_i$, i.e. $J_{post}^i(d)=J_{post}^i(0)$ if $0\leq d\leq \bar{d}_i$. The length $\bar{d}_i$ is denoted as the compression distance, i.e. the minimal distance that we have to move the pusher in order to touch the object. In the illustrative example of \prettyref{fig:push}(b), we have $\bar{d}_i=\bar{d}_{ia}+\bar{d}_{ib}+\bar{d}_{ic}$. $\bar{d}_i$ can be computed analytically by the recursive raycasting \prettyref{alg:dbar}.
If the pushing distance is larger than $\bar{d}_i$, then the distance $\E{dist}(\E{o}_i(d),\E{t}_j)$ will change according to the Voronoi region of $\E{t}_j$ that $\E{o}_i$ belongs \cite{mirtich1998v}. Within each Voronoi region, $\E{dist}(\E{o}_i(d),\E{t}_j)$ is a quadratic function of $d$. In the planar case, there are only two types of Voronoi regions, corresponding to vertex and edge, respectively. In the example of \prettyref{fig:push}(b), we illustrate three quadratic pieces with $J_{post}^{1a},J_{post}^{1c}$ corresponding to edge regions and $J_{post}^{1b}$ to a vertex region. The dividing points between regions can be determined by computing the intersections between Voronoi region boundaries and the object's moving path. 

We summarize our push planner \prettyref{alg:push} by estimating the computational complexity. For each pushing direction, our planner first enumerates possible pusher locations, where there are at most $16N^2$ cases. For each case, there are at most $N$ objects in the affected set. Each object contributes a piecewise quadratic cost model, with at most $1+2\fmax{j=1,\cdots,T}V(\E{t}_j)$ pieces, where $V(\E{t}_j)$ is the number of vertices of $\E{t}_j$. After summing up $J_{post}^i$ to get $J_{post}$, it has at most $N(1+2\fmax{j=1,\cdots,T}V(\E{t}_j))$ pieces. If we assume solving for each quadratic piece takes constant time, then our algorithm has the following complexity: $\mathcal{O}(16DN^3(1+2\fmax{j=1,\cdots,T}V(\E{t}_j)))$. Note that the complexity in practice is much lower than this upper bound because many pusher locations are intersecting objects and thus pruned.
\section{\label{sec:high}High-Level Receding-Horizon Planner}
We can perform synergetic planning using low-level planners alone, by first computing the one-step greedy actions, $\mathcal{G},\mathcal{P}$, and then picking the action with larger cost reduction. But this strategy has two drawbacks. First, our push planner is based on a simplified kinematic model, which might suffer from a high approximate error. Second, our cost model in low-level planner does not take transit cost into consideration, which can slow down the overall efficacy \cite{koga1994multi}. 

Our high-level planning \prettyref{alg:high_level} mitigates these two drawbacks. First, we use Box2D \cite{catto2010box2d} to simulate $\mathcal{P}$ and compute a more accurate cost reduction by solving \prettyref{eq:cost} before and after each simulation. Second, our high-level planner considers the transit cost $J_{transit}$ and seeks to maximize the following modified cost function:
\begin{align}
\label{eq:transit}
J_{rate}=(J(\E{o}_i)-J_{post})/J_{transit},
\end{align}
i.e. the rate of cost reduction per unit end-effector movement of the gripper. To effectively reduce $J_{rate}$, we have to expand the decision space. We observe that, when using low-level planners alone, the gripper will suffer from unnecessary transits by jumping between target regions (i.e. grasping an object to $t_i$, pushing another object to $t_j$, and then grasping a third object to $t_i$ again). To avoid this artifact, we choose to not use the overall greedy actions $\mathcal{G},\mathcal{P}$, but use the actions focused on a single target region, i.e. $\mathcal{G}_j,\mathcal{P}_j$. In other words, we allow the gripper to reduce the cost of one target region as much as possible, before transiting to another target region. In addition, our high-level planner seeks to reduce \prettyref{eq:transit} over multiple steps via an action-tree search. Whenever a push action is used in a tree node, then the state after the push is predicted using Box2D \cite{catto2010box2d} (\prettyref{ln:sim} of \prettyref{alg:high_level}). The branching factor of this search is $2T$, as we choose one action from the set $G_{1,\cdots,T},P_{1,\cdots,T}$. We can further reduce the branching factor by half if we choose between grasping or pushing actions greedily for each target region. This greediness does not degradate the overall performance empirically.
\section{\label{eq:results}Evaluations}
\begin{figure*}[t]
\vspace{-5px}
\centering
\includegraphics[width=0.99\textwidth]{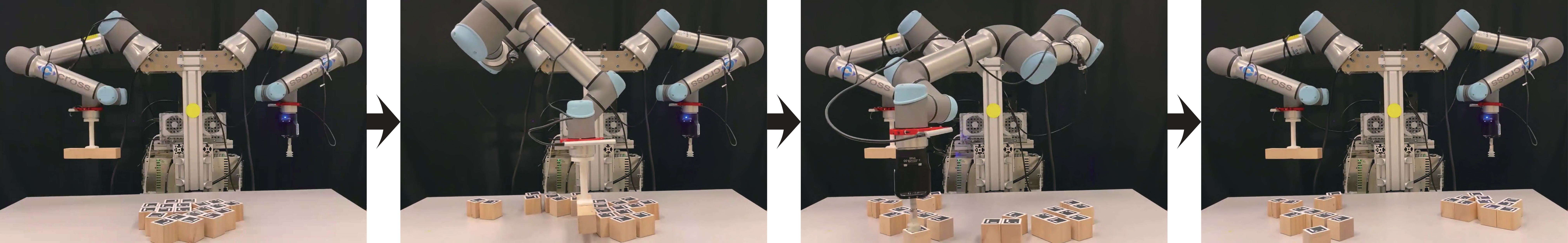}
\caption{\label{fig:TRINA} From left to right: initial configuration, pushing action, grasping action, final configuration.}
\vspace{-5px}
\end{figure*}
We implement our method using mixed Python/C++, where we use C++ to perform multi-threaded distance computations as in \prettyref{fig:push}. All the experiments are conducted on a desktop machine with a 10 core Intel(R) Xeon(R) W-2155 CPU. We evaluate our method in a simulated environment with a STAUBLI 6-axis Industrial robotic Arm TX90 as well as the bimanual hardware platform in \prettyref{fig:TRINA} equipped with a $8$cm-long pushing bar and Robotiq's vacuum grippers. The simulation is performed using ODE \cite{smith2005open} where the control signals are provided by a PID controller. Our goal is to push $N=50-200$ cubical bricks ($5\times5\times5cm^3$) to $T=1-4$ target regions ($100\times50cm^2$).

\begin{figure}[t]
\vspace{-5px}
\centering
\setlength\tabcolsep{1pt}
\renewcommand{\arraystretch}{0.0}
\begin{tabular}{cc}
\includegraphics[trim=5mm 0mm 5mm 5mm,clip,width=0.48\linewidth]{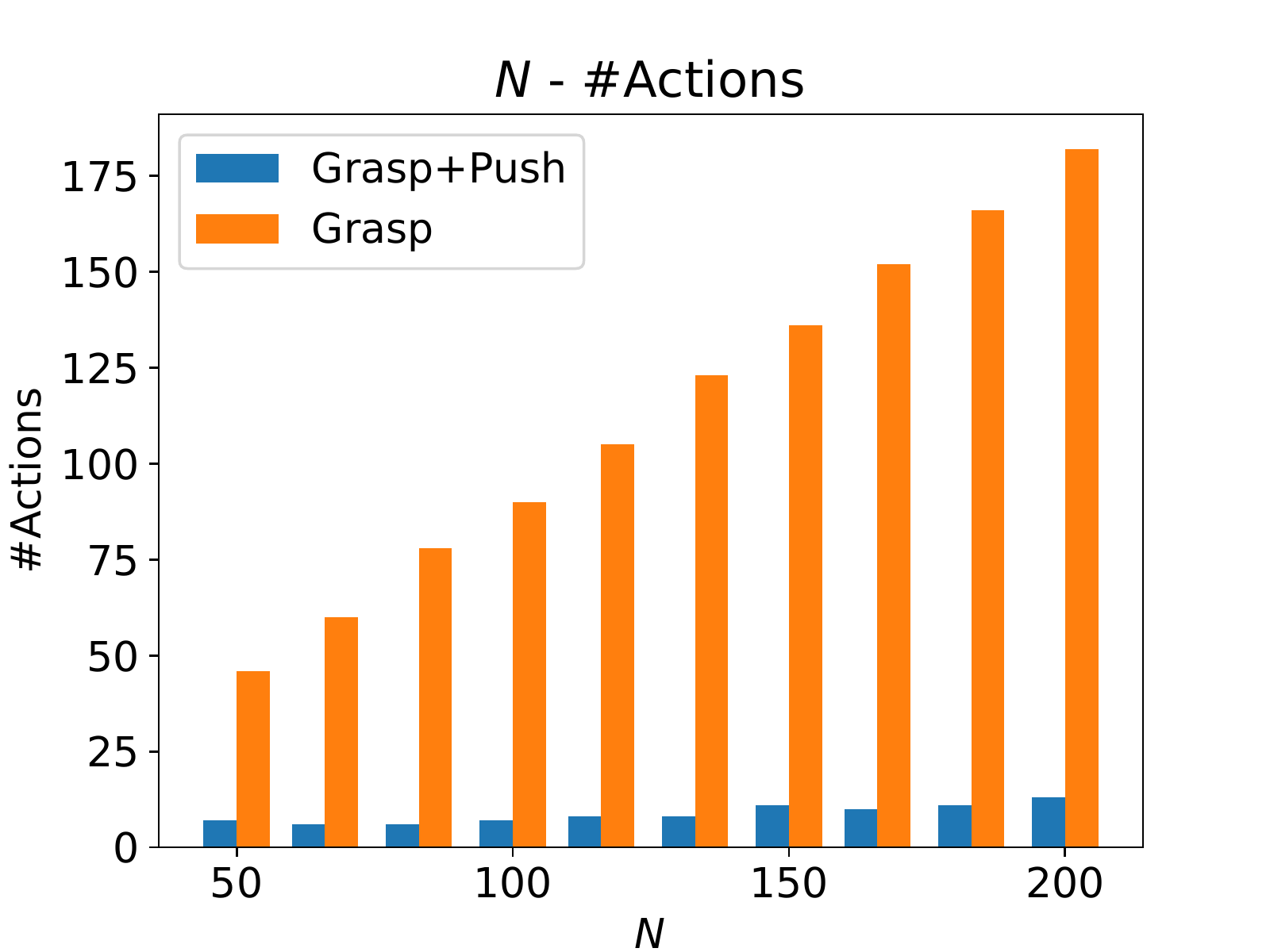} &
\includegraphics[trim=5mm 0mm 5mm 5mm,clip,width=0.48\linewidth]{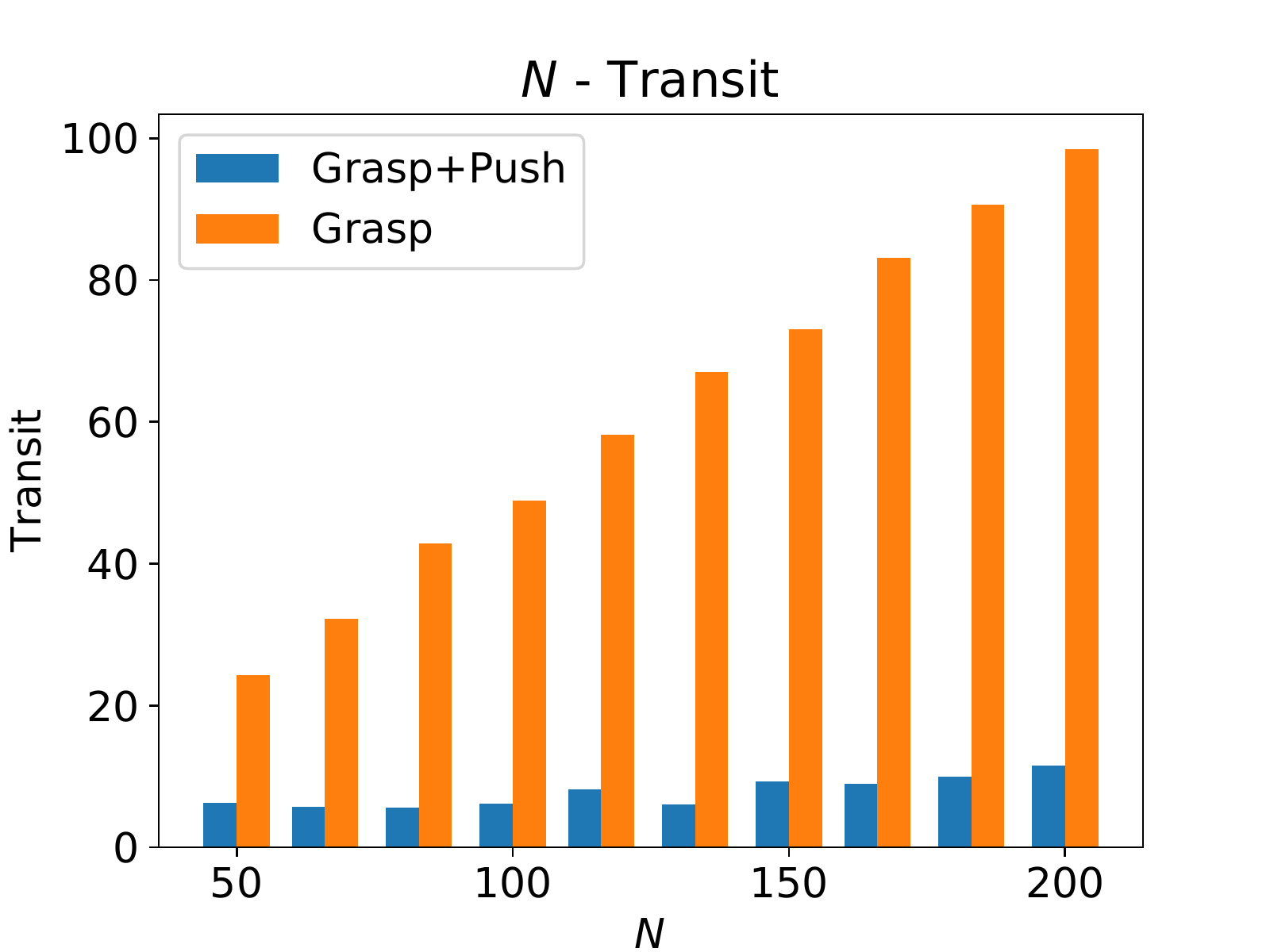}   \\
\includegraphics[trim=5mm 0mm 5mm 5mm,clip,width=0.48\linewidth]{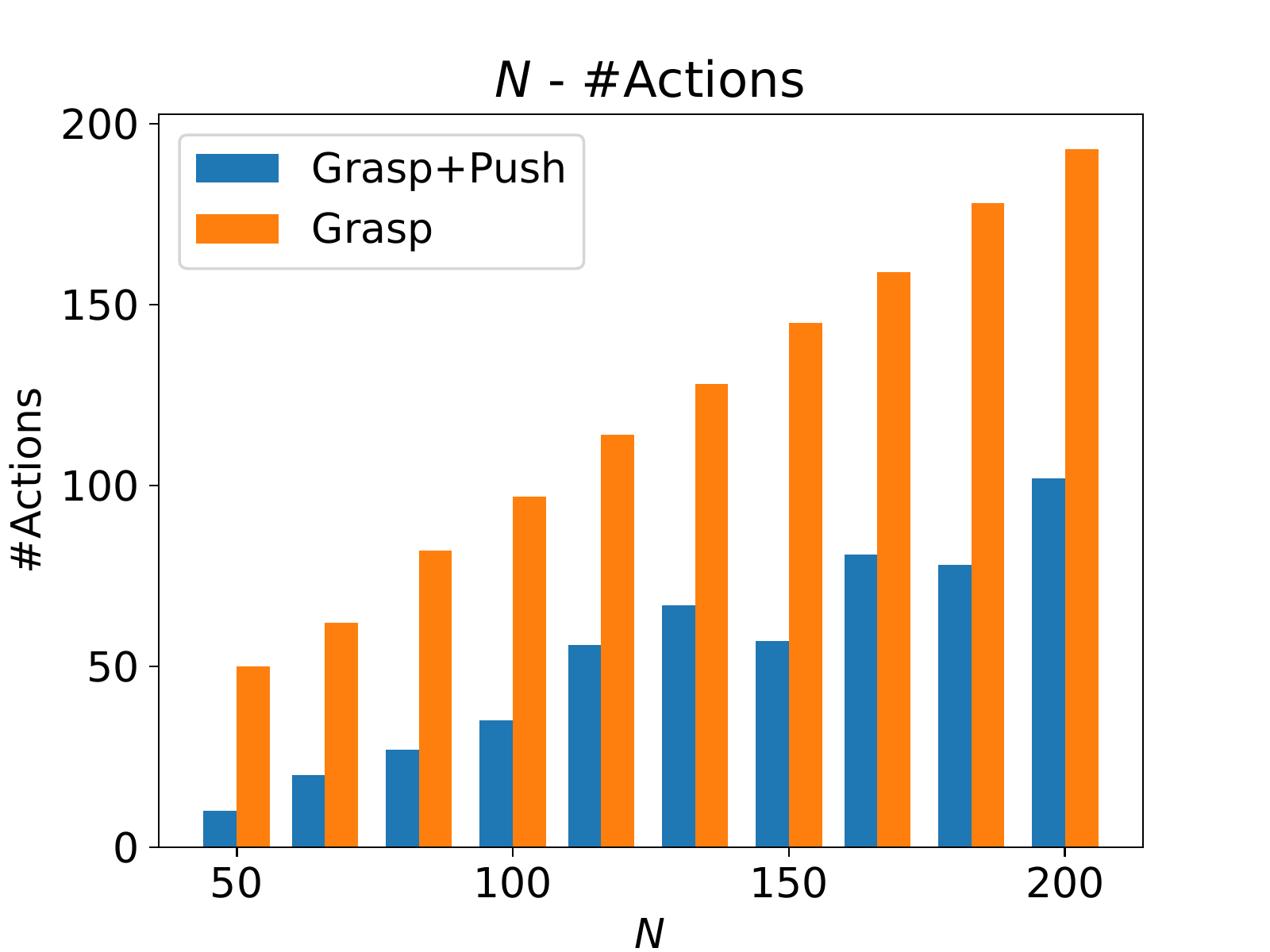} &
\includegraphics[trim=5mm 0mm 5mm 5mm,clip,width=0.48\linewidth]{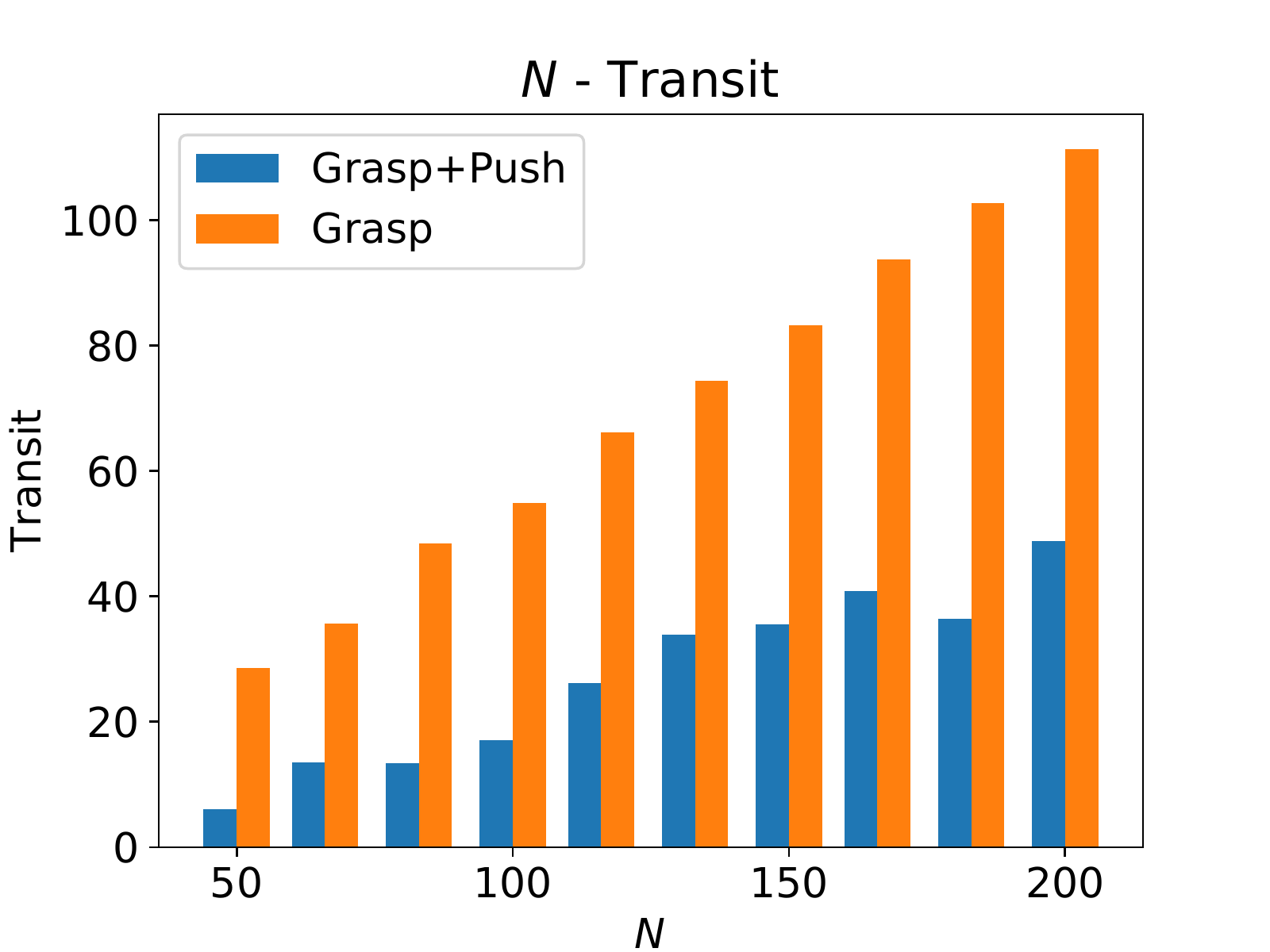}  \\
\end{tabular}
\caption{\label{fig:woPush} We show the speedup/number-of-actions/transit-cost with/without pushing actions, plotted against the number of objects. Top Row: $T=2$, $C=1$, $c_0(\E{t}_j)=N/2$; Bottom Row: $T=2$, $C=N$.}
\end{figure}
\TE{Speedup Using Pushing Actions:} In \prettyref{fig:woPush}, we show the speedup using grasping+pushing actions, as compared with grasping actions alone. We use two settings: unlabeled ($C=1$) and labeled categories ($C=N$). When $C=1$, the robot are mostly using pushing actions, and the pushing action could provide $6-15\times$ speedup in terms of number of actions and $3-10\times$ speedup in terms of the transit cost. When $C=N$, the robot is forced to use more grasping actions due to fixed assignment, and the pushing actions could only provide $1.4-3.1\times$ speedup in terms number of actions and $1.6-4.2\times$ speedup in terms of the transit cost. In each test case, the initial object poses are sampled randomly. Some of these objects are out of the reach for the gripper, so the number of grasps is not a perfect linear function of $N$.

\begin{figure}[ht]
\vspace{-5px}
\centering
\setlength\tabcolsep{1pt}
\renewcommand{\arraystretch}{0.0}
\begin{tabular}{cc}
\includegraphics[trim=5mm 0mm 5mm 5mm,clip,width=0.48\linewidth]{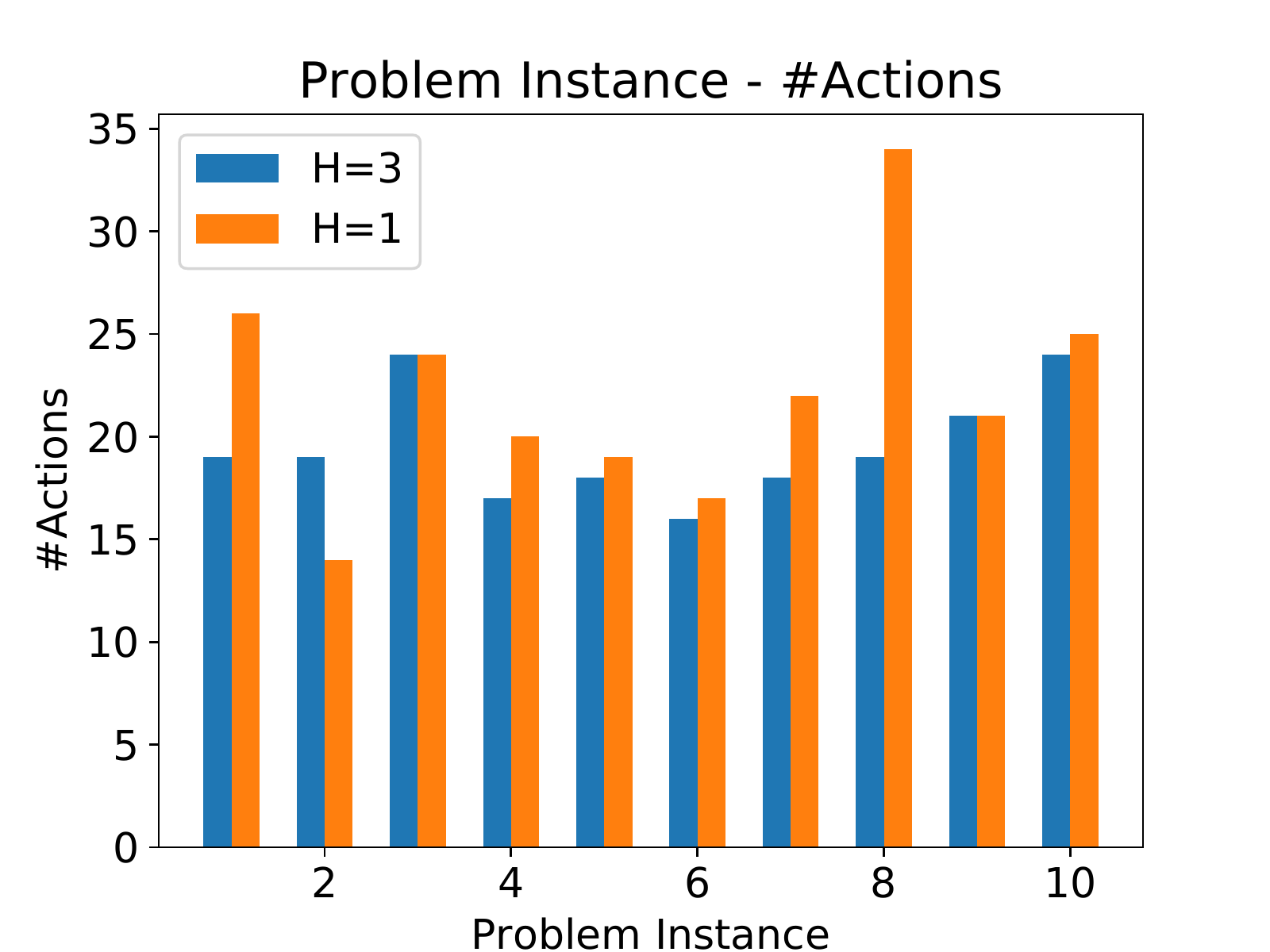} &
\includegraphics[trim=5mm 0mm 5mm 5mm,clip,width=0.48\linewidth]{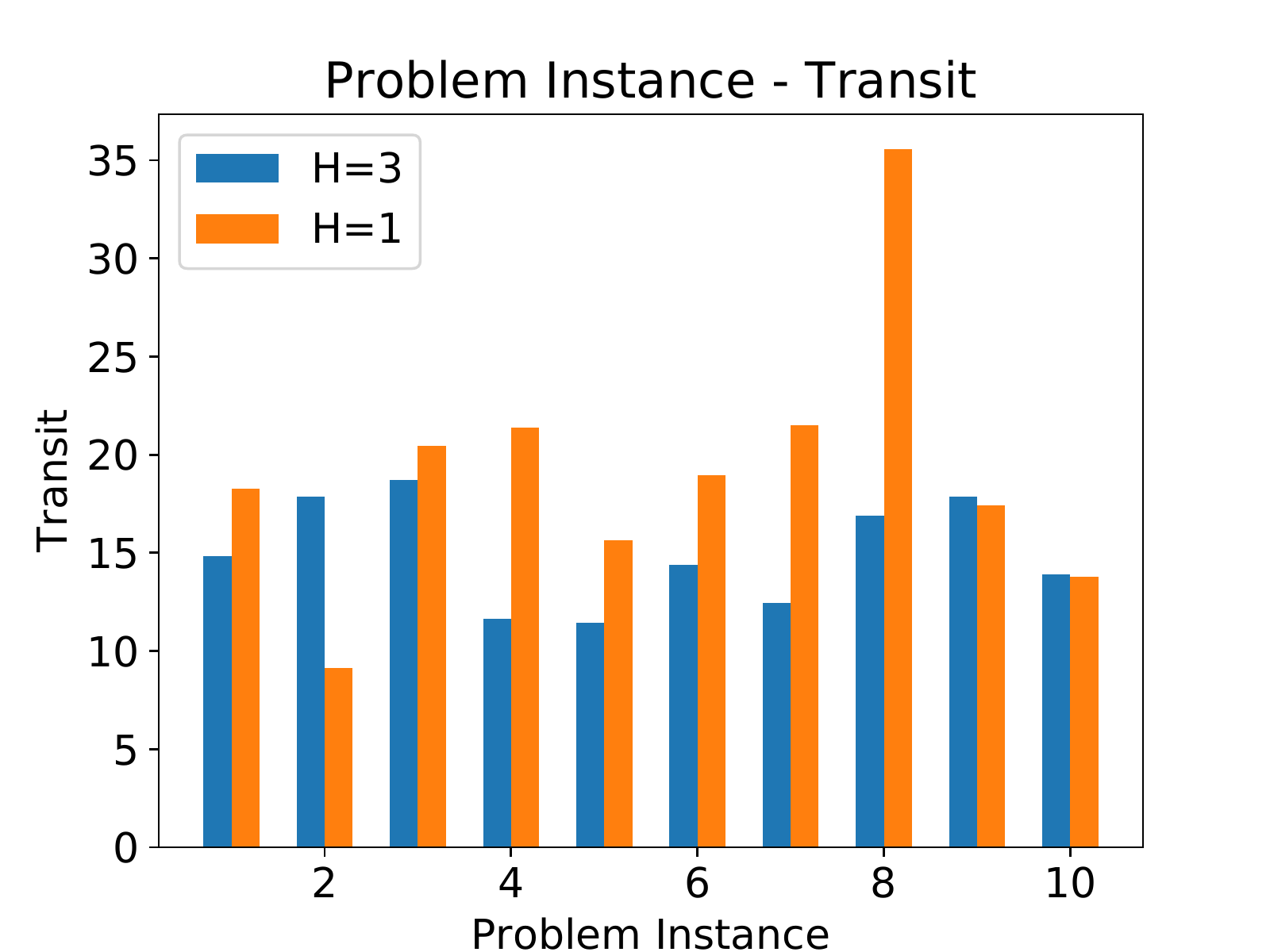}
\end{tabular}
\caption{\label{fig:woHighLevel} On 10 randomly generated sorting tasks, we show the speedup/number-of-actions/transit-cost with/without high-level planner ($H$ is the high-level planning horizon, $H=1$ means no high-level planner used).}
\end{figure}
\TE{Speedup Using The High-Level Planner:} In \prettyref{fig:woHighLevel}, we show the speedup with/without using the high-level planner. We randomly generate $10$ sorting tasks with parameters sampled uniformly in range: $50\leq N\leq200$, $2\leq T\leq4$, $C=2$, $c_0(\E{t}_j)=0.8N/2$, $c_1(\E{t}_j)=0.2N/2$. We observe that the high-level planner does not help reducing the number of actions except in one task. However, the high-level planner does reduce the transit cost in most cases, achieving at most $2\times$ speedup. Using $H=3$ increases the computational time of each decision making by $8s$ as compared with $H=1$. We observe that further increasing $H$ does not worth the extra computational time.

\section{\label{eq:conclusion}Conclusion \& Limitations}
We propose a synergetic push-grasp planner for large-scale object sorting tasks. Our planner uses the grasping action to ensure feasibility of the task, and we use pushing actions to accelerate the execution. We show that one-step greedy grasping actions can be found by solving MILP, and with the help of a simplified kinematic model, one-step greedy pushing actions can be found by analyzing and enumerating pusher configurations. Finally, we take the transit cost into consideration using a high-level planner to perform multi-step action selection. As a major advantage, our method is fully analytic and does not require any parameter tuning, as compared with prior learning-based methods \cite{Song2019,Zeng2018}.

Our method can be further improved in three ways. First, our method assumes perfect sensing and requires the exact knowledge of object configurations. In practice, objects can be occluded and thus cannot be localized exactly, in which case the two low-level planner should be modified to account for uncertainties. Second, our simplified kinematic model is similar to \cite{TerryRuss2020WAFR}, which assumes that characteristic length of each object is much smaller than that of the pusher or the target region size. If larger objects are sorted, our assumptions on object motions during pushing will be violated. Third, although we have shown that grasping actions are feasible for object sorting, the pushing actions can violate this guarantee. This is because objects might be pushed too far away to leave the reachable set of the gripper. In practice, a hardware-side or software-side safety mechanism can be implemented to bound the objects to the reachable set.
\section{Acknowledgement}
This work is partially funded by NSF Grant \#2025782.
\clearpage
\bibliographystyle{IEEEtranS}
\bibliography{references}
\ifarxiv
\section{\label{sec:feasible}Appendix: Feasibility Guarantee}
We call a sampled location $\E{p}_{mn}$ a buffer location if the following condition holds:
\begin{align*}
\E{p}_{mn}\notin\E{t}_j\wedge
\E{dist}(\E{p}_{mn},\E{t}_j)<\fmin{i\neq j}\E{dist}(\E{t}_i,\E{t}_j)\quad\forall j.
\end{align*}
The buffer location could be understood as a generalized center point of all the target regions, with a smaller distance to any region than any other region (dashed red region in \prettyref{fig:buffer}). Note that if $\E{dist}$ is the Euclidean distance, then buffer locations might not always be available depending on the positions of target regions. In these cases, we can simply pick any reachable and collision-free position and pretend it is a buffer location by setting all $\E{dist}(\E{p}_{mn},\E{t}_j)=0$ in \prettyref{eq:MILP}. With available buffer locations, it is unsurprising that grasping is feasible in solving most object sorting problems, as many prior works using only grasping actions. We formalize this result in the following lemma:
\begin{SCfigure}[][t]
\centering
\includegraphics[width=0.24\textwidth]{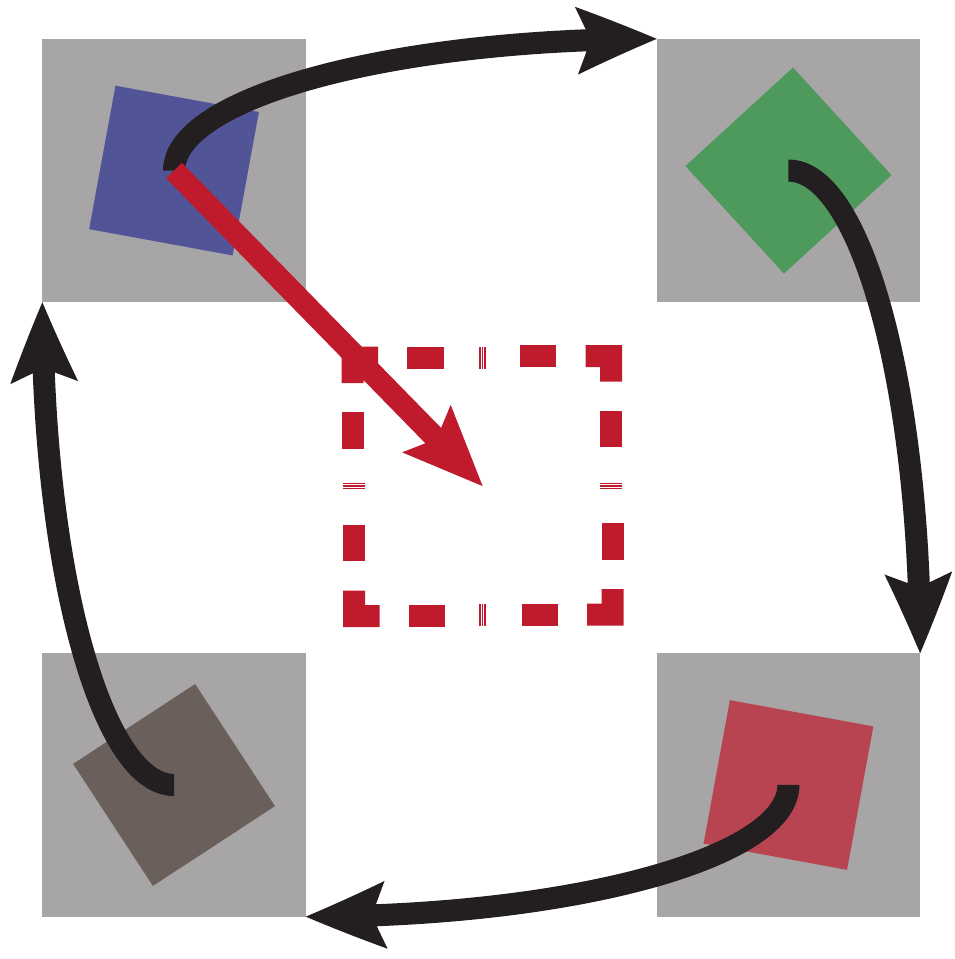}
\vspace{-5px}
\caption{\label{fig:buffer}\small{We have 4 target regions each occupied with an object, but these objects have to be moved to a neighboring region (black arrow). If a buffer location is available (dashed red), then one-step greedy grasping action is still feasible, because moving any object to the buffer location (red arrow) strictly reduce $J$.}}
\end{SCfigure}
\begin{lemma}
\label{lem:feas}
If all objects are reachable, i.e. $\E{reach}(\E{o}_i)=1$, enough sampled locations exist in each target region, i.e. $|\E{S}_j|>9\sum_{k=1}^C c_k(\E{t}_j)$ for all $j$, and a buffer location $\bar{\E{p}}_{mn}$ can always be found such that $\E{reach}(\bar{\E{p}}_{mn})=1$ and $B_R(\bar{\E{p}}_{mn})$ does not overlap any objects, then the object sorting task is feasible using one-step greedy grasping actions.
\end{lemma}
\begin{proof}
We prove by induction on the monotonic reduction of $J(\E{o}_i)$. \TE{Base Case:} If $J(\E{o}_i)=0$, then the task is feasible. Otherwise, we must have $b_{ij}=1$ for some $\E{o}_i\notin\E{t}_j$. We analyze this situation case-by-case. \TE{Induction Case I:} If $b_{ij}=1$ for all $\E{o}_i\in\E{t}_j$, then all the assignments are correct and $\E{o}_i\notin\E{t}_j$ implies that $\E{o}_i\notin\cup_{j=1}^T\E{t}_j$. Notice that the $\sqrt{2}R$-spacing ensures that, for each $\E{o}_i$, $B_R(\E{o}_i)$ will overlap at most 9 sampled locations. Therefore, $|\E{S}_j|>9\sum_{k=1}^C c_k(\E{t}_j)$ implies that there must be an sampled location $\E{p}_{mn}\in\E{S}_j$ such that $B_R(\E{p}_{mn})$ does not overlap any object. We can then move $\E{o}_i$ to $\E{p}_{mn}$ and strictly reduce $J(\E{o}_i)$. \TE{Induction Case II:} If $b_{ij}=1$ for some $\E{o}_i\in\E{t}_k$ and $k\neq j$, then we have a mis-assignment. We can strictly reduce $J(\E{o}_i)$ by moving $\E{o}_i$ to some $\E{p}_{mn}\in\E{S}_j$. If there are some sampled location in $\E{S}_j$ without overlaps, then this is feasible and we can strictly reduce $J(\E{o}_i)$. \TE{Induction Case III:} If $b_{ij}=1$ for some $\E{o}_i\in\E{t}_k$, $k\neq j$, and no sampled locations can be found in $\E{S}_j$ without overlaps, then we mark the situation as a $k\to j$ dependency. We can build a dependency chain $j_1\to j_2\to \cdots \to j_Q$ until one of two cases happens. \TE{Induction Case III-A:} If $j_Q\neq j_1$, then we have finally found a region $\E{t}_{j_Q}$ with some empty sampled location, we can move some object from $\E{t}_{j_{Q-1}}$ to $\E{t}_{j_Q}$ and strictly reduce $J(\E{o}_i)$. \TE{Induction Case III-B:} If $j_Q=j_1$, then we have found a loopy dependency. To resolve this loop, we need the buffer location $\bar{\E{p}}_{mn}$. By the definition of a buffer, it is obvious that moving any object in the loop to this buffer location will strictly reduce $J(\E{o}_i)$ (one-step greediness). Indeed, we can move objects cyclically along the augmented loop $j_1\to j_2\to,\cdots,j_{Q-1}\to\bar{\E{p}}_{mn}\to j_1$ to clear the buffer location while strictly reduce strictly reduce $J(\E{o}_i)$.
\end{proof}
This result guarantees the feasibility of grasping actions under mild assumptions on the size of target regions and the dispersion of $o_i$. \prettyref{lem:feas} also suggests a way to form the set of $\E{p}_{mn}$ used when solving \prettyref{eq:MILP}, i.e. $\E{p}_{mn}$ is a subset of $\cup_{j=1}^T\E{S}_j$ without overlapping any $B_R(\E{o}_i)$ plus a buffer location. In practice, the problem can always be solved and we never observed the need of buffer locations.
\fi
\end{document}